\documentclass[12pt]{article}
\usepackage{amsmath,amsfonts,amsthm,mathrsfs}

\newtheorem{remark}{Remark}
\newtheorem{prop}{Proposition}
\newtheorem{corollary}{Corollary}
\newtheorem{definition}{Definition}
\newtheorem{lemma}{Lemma}
\newtheorem{theorem}{Theorem}
\newtheorem{example}{Example}

\newcommand{\E}{{\mathbb E}}

\newcommand{\RR}{\mathbb R}

\newcommand{\NN}{\mathbb N}

\def\si{\sigma}

\numberwithin{equation}{section}

\def\bz{\mathbf z}
\def\Var{{\bf var}}

\title{Learning Theory Approach to Minimum Error Entropy Criterion$^\dag$\footnotetext{\dag
~The work described in this paper is supported by National Science Foundation of China under Grant 11201348 and by a grant
from the Research Grants Council of Hong Kong [Project No. CityU 103709]. Ting Hu (tinghu@whu.edu.cn) is with School of
Mathematics and Statistics, Wuhan University, Wuhan 430072, China. Jun Fan (junfan2@student.cityu.edu.hk) and Ding-Xuan
Zhou (mazhou@cityu.edu.hk) are with Department of Mathematics, City University of Hong Kong, Kowloon, Hong Kong, China.
Qiang Wu (wuqiangmath@gmail.com) is with Department of Mathematical Sciences, Middle Tennessee State University, Box 34
Murfreesboro, TN 37132-0001, USA.}}

\author{Ting Hu, Jun Fan, Qiang Wu, and Ding-Xuan Zhou}

\date{}

\begin{document}
\maketitle

\begin{abstract}
We consider the minimum error entropy (MEE) criterion and an empirical risk minimization learning algorithm when an
approximation of R\'enyi's entropy (of order $2$) by Parzen windowing is minimized. This learning algorithm involves a
Parzen windowing scaling parameter. We present a learning theory approach for this MEE algorithm in a regression setting
when the scaling parameter is large. Consistency and explicit convergence rates are provided in terms of the approximation
ability and capacity of the involved hypothesis space. Novel analysis is carried out for the generalization error
associated with R\'enyi's entropy and a Parzen windowing function, to overcome technical difficulties arising from the
essential differences between the classical least squares problems and the MEE setting. An involved symmetrized least
squares error is introduced and analyzed, which is related to some ranking algorithms.
\end{abstract}


\section{Introduction}

Information theoretical learning is inspired by introducing
information theory into a machine learning paradigm. Within this
framework algorithms have been developed for several learning tasks,
including regression, classification, and unsupervised learning. It
attracts more and more attention because of its successful
applications in signal processing, system engineering, and data
mining. A systematic treatment and recent development of this area
can be found in \cite{MEEbook} and references therein.

Minimum error entropy (MEE) is a principle of information
theoretical learning and provides a family of supervised learning
algorithms. It was introduced for adaptive system training in
\cite{ErdPri02} and has been applied to blind source separation,
maximally informative subspace projections, clustering, feature
selection, blind deconvolution, and some other topics
\cite{ErdPri03, MEEbook, Silva}. The idea of MEE is to extract
from data as much information as possible about the data
generating systems by minimizing error entropies in various ways.
In information theory, entropies are used to measure average
information quantitatively. For a random variable $E$ with
probability density function \label{pdfE} $p_E$, Shannon's entropy
of $E$ \label{Shannon entropy} is defined as
$$H_S(E) = -\E[\log p_E]= -\int p_E(e) \log p_E(e) d e $$
while R\'enyi's entropy \label{Renyi entropy} of order $\alpha$
($\alpha >0$ but $\alpha \not=1$) is defined as
$$H_{R, \alpha} (E)= \frac{1}{1-\alpha} \log \E[p_E^{\alpha-1}] = \frac{1}{1-\alpha} \log\left(\int (p_E(e))^{\alpha } d e\right) $$
satisfying $\lim_{\alpha \to 1} H_{R, \alpha} (E) = H_S(E)$. In
supervised learning our target is to predict the response variable
\label{response variable} $Y$ from the explanatory variable
\label{explanatory variable} $X$. Then the random variable $E$
becomes the error variable $E=Y-f(X)$ \label{random variable} when
a predictor $f(X)$ is used and the MEE principle aims at searching
for a predictor $f(X)$ that contains the most information of the
response variable by minimizing information entropies of the error
variable $E=Y-f(X)$. This principle is a substitution of the
classical least squares method when the noise is non-Gaussian.
Note that $\E [Y-f(X)]^2 = \int e^2 p_E(e) d e$. The least squares
method minimizes the variance of the error variable $E$ and is
perfect to deal with problems involving Gaussian noise (such as
some from linear signal processing). But it only puts the first
two moments into consideration, and does not work very well for
problems involving heavy tailed non-Gaussian noise. For such
problems, MEE might still perform very well in principle since
moments of all orders of the error variable are taken into account
by entropies. Here we only consider R\'enyi's entropy \label{Renyi
entropy 2} of order $\alpha =2$: $H_{R}(E) = H_{R, 2}(E) =- \log
\int (p_E(e))^2 d e$. Our analysis does not apply to R\'enyi's
entropy of order $\alpha\not= 2$.

 In most real applications, neither the explanatory variable $X$ nor the response variable $Y$ is explicitly known.
 Instead, in supervised learning, a sample \label{sample} $\bz=\{(x_i, y_i)\}_{i=1}^{m}$ is available which reflects the distribution
 of the explanatory variable $X$ and the functional relation between $X$ and the response variable $Y$. With this sample, information entropies
 of the error variable $E=Y-f(X)$ can be approximated by estimating its probability density function $p_E$
 by Parzen \cite{Parzen} windowing \label{Parzen windowing} $\widehat{p}_E (e) =\frac{1}{m h} \sum_{i=1}^m G(\frac{(e-e_i)^2}{2 h^2})$,
 where $e_i=y_i-f(x_i)$, $h>0$ is an \label{MEE scaling parameter} MEE scaling parameter, and $G$ is a \label{windowing function} windowing function.
 A typical choice for the windowing function $G(t) =\exp \{-t\}$ corresponds to Gaussian windowing.
Then approximations of Shannon's entropy and R\'enyi's entropy of
order $2$ are given by their empirical versions $-\frac{1}{m}
\sum_{i=1}^m \log \widehat{p}_E (e_i)$ and $-\log (\frac{1}{m}
\sum_{i=1}^m \widehat{p}_E (e_i))$ as \label{empirical Shannon
entropy}
$$
\widehat{H_S} = -  \frac{1}{m}\sum_{i=1}^m \log\left[\frac 1{m h} \sum_{j=1}^m G\left(\frac{(e_i-e_j)^2}{2 h^2}\right)\right]
$$
and \label{empirical Renyi entropy}
$$
\widehat{H_R} = - \log \frac{1}{m^2 h} \sum_{i=1}^m \sum_{j=1}^m G\left(\frac{(e_i-e_j)^2}{2 h^2}\right),
$$
respectively.
The empirical MEE is implemented by minimizing these computable quantities.

Though the MEE principle has been proposed for a decade and MEE algorithms have been shown to be effective in various
applications, its theoretical foundation for mathematical error analysis is not well understood yet. There is even no
consistency result in the literature. It has been observed in applications that the scaling parameter $h$ should be large
enough for MEE algorithms to work well before smaller values are tuned. However, it is well known that the convergence of
Parzen windowing requires $h$ to converge to $0.$ We believe this contradiction imposes difficulty for rigorous
mathematical analysis of MEE algorithms. Another technical barrier for mathematical analysis of MEE algorithms for
regression is the possibility that the regression function may not be a minimizer of the associated generalization error,
as described in detail in Section \ref{DiffNovelty} below. The main contribution of this paper is a consistency result for
an MEE algorithm for regression. It does require $h$ to be large and explains the effectiveness of the MEE principle in
applications.

In the sequel of this paper, we consider an MEE learning algorithm
that minimizes the empirical R\'enyi's entropy $\widehat{H_R}$ and
focus on the regression problem. We will take a learning theory
approach and analyze this algorithm in an {\it empirical risk
minimization} (ERM) setting. Assume $\rho$ is a probability
measure on ${\mathcal Z}:={\mathcal X} \times {\mathcal Y}$, where
${\mathcal X}$ is a separable metric space (input space for
learning) and ${\mathcal Y}=\RR$ (output space). Let $\rho_X$ be
its marginal distribution on ${\mathcal X}$ (for the explanatory
variable $X$) and $\rho(\cdot|x)$ be the conditional distribution
of $Y$ for given $X=x$. The sample $\bz$ is assumed to be drawn
from $\rho$ independently and identically distributed. The aim of
the regression problem is to predict the conditional mean of $Y$
for given $X$ by learning the regression function
\label{regression function} defined by
$$ f_\rho(x)=\E(Y|X=x)=\int_{{\mathcal X}} yd\rho(y|x), \qquad x\in {\mathcal X}.$$

The minimization of empirical R\'enyi's entropy cannot be done over all possible measurable functions which would lead to
overfitting. A suitable hypothesis space should be chosen appropriately in the ERM setting. The ERM framework for MEE
learning is defined as follows. Recall $e_i =y_i - f(x_i)$.

\begin{definition}
Let $G$ be a continuous function \label{windowing function MEE}
defined on $[0, \infty)$ and \label{MEE scaling parameter 2}
$h>0$. Let ${\mathcal H}$ be a compact subset of $C({\mathcal
X})$. Then the MEE learning algorithm associated with ${\mathcal
H}$ is defined by \label{MEE algorithm}
\begin{equation}\label{Etf}
f_{\bf z} = \arg \min_{f\in {\mathcal H}} \left\{- \log \frac{1}{m^2 h} \sum_{i=1}^m \sum_{j=1}^m
G\left(\frac{\left[\left(y_i - f(x_i)\right) -\left(y_j - f(x_j)\right)\right]^2}{2 h^2}\right)\right\}.
\end{equation}
\end{definition}

The set ${\mathcal H}$ is called the hypothesis space
\label{hypothesis space} for learning. Its compactness ensures the
existence of a minimizer $f_{\bf z}$. Computational methods for
solving optimization problem (\ref{Etf}) and its applications in
signal processing have been described in a vast MEE literature
\cite{MEEbook, ErdPri02, ErdPri03, Silva}. For different purposes
the MEE scaling parameter $h$ may be chosen to be large or small.
It has been observed empirically that the MEE criterion has nice
convergence properties when the MEE scaling parameter $h$ is
large. The main purpose of this paper is to verify this
observation in the ERM setting and show that $f_{\bf z}$ with a
suitable constant adjustment approximates the regression function
well with confidence. Note that the requirement of a constant
adjustment is natural because any translate $f_{\bf z} +c$ of a
solution $f_{\bf z}$ to (\ref{Etf}) with a constant $c\in \RR$ is
another solution to (\ref{Etf}). So our consistency result for MEE
algorithm (\ref{Etf}) will be stated in terms of the variance
$\Var[f_{\bf z}(X) -f_\rho(X)]$ of the error function $f_{\bf z}
-f_\rho$. Here we use $\Var$ to denote the variance of a random
variable. \label{variance}

\section{Main Results on Consistency and Convergence
Rates}\label{mainresults}

Throughout the paper, we assume $h \geq 1$ and that \label{power
q}
\begin{equation}\label{assummoment}
\E[|Y|^q]<\infty \ \hbox{for some} \ q>2, \ \hbox{and} \ f_\rho
\in L^\infty_{\rho_X}. \quad \hbox{Denote} \ q^* =\min\{q-2, 2\}.
\end{equation}
We also assume that the windowing function $G$ \label{windowing
function constant} satisfies \label{decay constant}
\begin{equation}\label{assumpG}
G \in C^2 [0, \infty), \ G_+'
(0) =-1, \ \hbox{and} \ C_G:= \sup_{t\in (0, \infty)} \left\{|(1
+t) G' (t)| + |(1 +t) G'' (t)|\right\}<\infty.
\end{equation}

The special example $G(t) = \exp\{-t\}$ for the Gaussian windowing
satisfies (\ref{assumpG}).

Consistency analysis for regression algorithms is often carried out
in the literature under a decay assumption for $Y$ such as uniform
boundedness and exponential decays. A recent study \cite{AC} was
made under the assumption $\E[|Y|^4]<\infty$. Our assumption
(\ref{assummoment}) is weaker since $q$ may be arbitrarily close to
$2$. Note that (\ref{assummoment}) obviously holds when $|Y| \leq M$
almost surely for some constant $M>0$, in which case we shall denote
$q^* =2$.

Our consistency result, to be proved in Section \ref{mainproof},
asserts that when $h$ and $m$ are large enough, the error
$\Var[f_{\bf z}(X) -f_\rho(X)]$ of MEE algorithm (\ref{Etf}) can
be arbitrarily close to the approximation error \cite{SZappr} of
the hypothesis space ${\mathcal H}$ with respect to the regression
function $f_\rho$.

\begin{definition}
The approximation error \label{approximation error} of the pair
$({\mathcal H}, \rho)$ is defined by
\begin{equation}\label{approxerror}
{\mathcal D}_{\mathcal H} (f_\rho) = \inf_{f \in {\mathcal
H}}\Var[f(X)-f_\rho(X)].
\end{equation}
\end{definition}

\begin{theorem}\label{consistthm}
Under assumptions (\ref{assummoment}) and (\ref{assumpG}), for any
$0< \epsilon \leq 1$ and $0 < \delta <1$, there exist
$h_{\epsilon, \delta} \geq 1$ and $m_{\epsilon, \delta}(h) \geq 1$
both depending on ${\mathcal H}, G, \rho, \epsilon, \delta$ such
that for $h \geq h_{\epsilon, \delta}$ and $m \geq m_{\epsilon,
\delta}(h)$, with confidence $1-\delta$, we have
\begin{equation}\label{consistbd}
\Var[f_{\bf z}(X) -f_\rho(X)] \leq {\mathcal D}_{\mathcal H}
(f_\rho) + \epsilon.
\end{equation}
\end{theorem}

Our convergence rates will be stated in terms of the approximation
error and the capacity of the hypothesis space ${\mathcal H}$
measured by covering numbers in this paper.

\begin{definition}
For $\varepsilon>0,$ the \label{covering number} {\it covering
number} ${\mathcal N}\left({\mathcal H}, \varepsilon\right)$ is
defined to be the smallest integer $l\in \mathbb{N}$ such that
there exist $l$ disks in $C({\mathcal X})$ with radius
$\varepsilon$ and centers in ${\mathcal H}$ covering the set
${\mathcal H}.$ We shall assume that for some constants $p>0$
\label{covering index} and $A_p
>0$, there holds
\begin{equation}\label{covercond} \log {\mathcal N}\left({\mathcal H},
\varepsilon\right) \leq A_{p}\varepsilon^{-p}, \qquad \forall
\varepsilon >0.
\end{equation}
\end{definition}

The behavior (\ref{covercond}) of the covering numbers is typical in learning theory. It is satisfied by balls of Sobolev
spaces on ${\mathcal X} \subset \RR^n$ and reproducing kernel Hilbert spaces associated with Sobolev smooth kernels. See
\cite{AB, Zhoucov, Zhoucap, Yao}. We remark that empirical covering numbers might be used together with concentration
inequalities to provide shaper error estimates. This is however beyond our scope and for simplicity we adopt the the
covering number in $C({\mathcal X})$ throughout this paper.

The following convergence rates for (\ref{Etf}) with large $h$
will be proved in Section \ref{mainproof}.

\begin{theorem}\label{mainresult}
Assume (\ref{assummoment}), (\ref{assumpG}) and covering number condition (\ref{covercond}) for some $p>0$. Then for any
$0< \eta \leq 1$ and $0 < \delta <1$, with confidence $1-\delta$ we have
\begin{equation}\label{mainbd}
\Var[f_{\bf z}(X) -f_\rho(X)] \leq \widetilde{C}_{\mathcal H}
\eta^{(2-q)/2} \left(h^{-\min\{q-2, 2\}} + h
m^{-\frac{1}{1+p}}\right) \log \frac{2}{\delta} + (1 + \eta)
{\mathcal D}_{\mathcal H} (f_\rho).
\end{equation}
If $|Y| \leq M$ almost surely for some $M>0$, then with confidence
$1-\delta$ we have
\begin{equation}\label{mainbdM}
\Var[f_{\bf z}(X) -f_\rho(X)] \leq \frac{\widetilde{C}_{\mathcal
H}}{\eta} \left(h^{-2} + m^{-\frac{1}{1+p}}\right) \log
\frac{2}{\delta} + (1 + \eta) {\mathcal D}_{\mathcal H} (f_\rho).
\end{equation}
Here $\widetilde{C}_{\mathcal H}$ is a constant independent of $m,
\delta, \eta$ or $h$ (depending on ${\mathcal H}, G, \rho$ given
explicitly in the proof).
\end{theorem}

\begin{remark}
In Theorem \ref{mainresult}, we use a parameter $\eta >0$ in error bounds (\ref{mainbd}) and (\ref{mainbdM}) to show that
the bounds consist of two terms, one of which is essentially the approximation error ${\mathcal D}_{\mathcal H} (f_\rho)$
since $\eta$ can be arbitrarily small. The reader can simply set $\eta =1$ to get the main ideas of our analysis.
\end{remark}

If moment condition (\ref{assummoment}) with $q \geq 4$ is
satisfied and $\eta =1$, then by taking $h = m^{\frac{1}{3(1 +
p)}}$, (\ref{mainbd}) becomes
\begin{equation}\label{specialchoice}\Var[(f_{\bf z}(X)
-f_\rho(X)] \leq 2 \widetilde{C}_{\mathcal H}
\left(\frac{1}{m}\right)^{\frac{2}{3(1 + p)}}\log \frac{2}{\delta}
+ 2 {\mathcal D}_{\mathcal H} (f_\rho).
\end{equation}
If $|Y| \leq M$ almost surely, then by taking
$h=m^{\frac{1}{2(1+p)}}$ and $\eta=1$, error bound (\ref{mainbdM})
becomes
\begin{equation}\label{specialchoiceM}
\Var[f_{\bf z}(X) -f_\rho(X)] \leq 2\widetilde{C}_{\mathcal H}
m^{-\frac{1}{1+p}} \log \frac{2}{\delta} + 2{\mathcal D}_{\mathcal
H} (f_\rho).
\end{equation}

\begin{remark}
When the index $p$ in covering number condition (\ref{covercond})
is small enough (the case when ${\mathcal H}$ is a finite ball of
a reproducing kernel Hilbert space with a smooth kernel), we see
that the power indices for the sample error terms of convergence
rates (\ref{specialchoice})
 and (\ref{specialchoiceM}) can be arbitrarily close to $2/3$ and
 $1$, respectively. There is a gap in the rates between the case of (\ref{assummoment}) with large
 $q$ and the uniform bounded case. This gap is caused by the
 Parzen windowing process for which our method does not lead to better estimates when $q>4$.
 It would be interesting to know whether the gap can be
 narrowed.
\end{remark}

Note the result in Theorem \ref{mainresult} does not guarantee
that $f_\bz$ itself approximates $f_\rho$ well when the bounds are
small. Instead a constant adjustment is required. Theoretically
the best constant is $\E[f_\bz(X)-f_\rho(X)]$. In practice it is
usually approximated by the sample mean $\frac 1m \sum_{i=1}^m
(f_\bz(x_i)-y_i)$ in the case of uniformly bounded noise and the
approximation can be easily handled. To deal with heavy tailed
noise, we project the output values onto the closed interval
$[-\sqrt{m}, \sqrt{m}]$ by the \label{projection} projection
$\pi_{\sqrt{m}}: \RR \to \RR$ defined by
$$ \pi_{\sqrt{m}} (y) =\left\{\begin{array}{ll} y, & \hbox{if} \ y
\in [-\sqrt{m}, \sqrt{m}], \\
\sqrt{m}, & \hbox{if} \ y
>\sqrt{m}, \\
-\sqrt{m}, & \hbox{if} \ y <-\sqrt{m}, \end{array}\right. $$ and then approximate $\E[f_\bz(X)-f_\rho(X)]$ by the
computable quantity
\begin{equation}\label{meanapprox}
\frac{1}{m} \sum_{i=1}^m \left[f_{\bf z}(x_i) - \pi_{\sqrt{m}} (y_i)\right].
\end{equation}
The following quantitative result, to be proved in Section \ref{mainproof}, tells us that this is a good approximation.

\begin{theorem}\label{meanesti}
Assume $\E[|Y|^2]<\infty$ and covering number condition
(\ref{covercond}) for some $p>0$. Then for any $0 < \delta <1$, with
confidence $1-\delta$ we have
\begin{equation}\label{mainmeanesti}
\sup_{f\in {\mathcal H}} \left|\frac{1}{m} \sum_{i=1}^m
\left[f(x_i) - \pi_{\sqrt{m}} (y_i)\right] - \E [f(X)
-f_\rho(X)]\right| \leq  \widetilde{C}'_{\mathcal H}
m^{-\frac{1}{2+p}} \log \frac{2}{\delta}
\end{equation}
which implies in particular that
\begin{equation}\label{mainfz}
\left|\frac{1}{m} \sum_{i=1}^m \left[f_{\bf z}(x_i) -
\pi_{\sqrt{m}} (y_i)\right] - \E [f_{\bf z}(X) -f_\rho(X)]\right|
\leq \widetilde{C}'_{\mathcal H} m^{-\frac{1}{2+p}} \log
\frac{2}{\delta},
\end{equation}
where $\widetilde{C}'_{\mathcal H}$ is the constant given by
$$ \widetilde{C}'_{\mathcal H} = 7\sup_{f\in {\mathcal
H}} \|f\|_\infty + 4 + 7 \sqrt{\E [|Y|^2]} + \E [|Y|^2] + A_p^{\frac{1}{2+p}}. $$
\end{theorem}

Replacing the mean $\E[f_\bz(X)-f_\rho(X)]$ by the quantity
(\ref{meanapprox}), we define an estimator \label{regression
estimator} of $f_\rho$ as
\begin{equation}\label{estimatorT}
\widetilde{f}_{\bf z} = f_{\bf z} -\frac{1}{m} \sum_{i=1}^m
\left[f_{\bf z}(x_i) - \pi_{\sqrt{m}} (y_i)\right].
\end{equation}
Putting (\ref{mainfz}) and the bounds from Theorem
\ref{mainresult} into the obvious error expression
\begin{equation}\label{mainmeanestimator}
 \left\|\widetilde{f}_{\bf z} -f_\rho\right\|_{L^2_{\rho_X}}
\leq \left|\frac{1}{m} \sum_{i=1}^m \left[f_{\bf z}(x_i) -
\pi_{\sqrt{m}} (y_i)\right] - \E [f_{\bf z}(X) -f_\rho(X)]\right|
+ \sqrt{\Var[(f_{\bf z}(X) -f_\rho(X)]},
\end{equation}
we see that $\widetilde{f}_{\bf z}$ is a good estimator of
$f_\rho$: the power index $\frac{1}{2+p}$ in (\ref{mainfz}) is
greater than $\frac{1}{2(1+p)}$, the power index appearing in the
last term of (\ref{mainmeanestimator}) when the variance term is
bounded by (\ref{specialchoiceM}), even in the uniformly bounded
case.

To interpret our main results better we present a corollary and an example below.

If there is a constant $c_\rho$ such that $f_\rho + c_\rho \in
{\mathcal H}$, we have ${\mathcal D}_{\mathcal H} (f_\rho) =0$. In
this case, the choice $\eta =1$ in Theorem \ref{mainresult} yields
the following learning rate. Note that (\ref{assummoment}) implies
$\E[|Y|^2]<\infty$.

\begin{corollary}
Assume (\ref{covercond}) with some $p>0$ and $f_\rho + c_\rho \in {\mathcal H}$ for some constant $c_\rho \in \RR$. Under
conditions (\ref{assummoment}) and (\ref{assumpG}), by taking $h =m^{\frac{1}{(1+p) \min\{q-1, 3\}}}$, we have with
confidence $1-\delta$,
$$ \left\|\widetilde{f}_{\bf z} -f_\rho\right\|_{L^2_{\rho_X}} \leq \left(\widetilde{C}'_{\mathcal H} + \sqrt{2 \widetilde{C}_{\mathcal H}}\right) m^{-
\frac{\min\{q-2, 2\}}{2(1+p) \min\{q-1, 3\}}} \log
\frac{2}{\delta}. $$ If $|Y| \leq M$ almost surely, then by taking
$h =m^{\frac{1}{2(1+p)}}$, we have with confidence $1-\delta$,
$$\left\|\widetilde{f}_{\bf z} -f_\rho\right\|_{L^2_{\rho_X}} \leq \left(\widetilde{C}'_{\mathcal H} + \sqrt{2 \widetilde{C}_{\mathcal H}}\right) m^{-
\frac{1}{2(1+p)}} \log \frac{2}{\delta}. $$
\end{corollary}

This corollary states that $\widetilde{f}_{\bf z}$ can approximate
the regression function very well. Note, however, this happens
when the hypothesis space is chosen appropriately and the
parameter $h$ tends to infinity.

A special example of the hypothesis space is a ball of a Sobolev
space $H^s ({\mathcal X})$ with index $s > \frac{n}{2}$ on a
domain ${\mathcal X} \subset \RR^n$ which satisfies
(\ref{covercond}) with $p =\frac{n}{s}$. When $s$ is large enough,
the positive index $\frac{n}{s}$ can be arbitrarily small. Then
the power exponent of the following convergence rate can be
arbitrarily close to $\frac{1}{3}$ when $\E[|Y|^4] < \infty$, and
$\frac{1}{2}$ when $|Y| \leq M$ almost surely.

\begin{example}
Let ${\mathcal X}$ be a bounded domain of $\RR^n$ with Lipschitz boundary. Assume $f_\rho \in H^s (X)$ for some $s>
\frac{n}{2}$ and take ${\mathcal H} =\{f\in H^s (X): \|f\|_{H^s (X)} \leq R\}$ with $R \geq \|f_\rho\|_{H^s (X)}$ and $R
\geq 1$. If $\E[|Y|^4] < \infty$, then by taking $h =m^{\frac{1}{3(1+n/s)}}$, we have with confidence $1-\delta$,
$$ \left\|\widetilde{f}_{\bf z} -f_\rho\right\|_{L^2_{\rho_X}} \leq C_{s, n, \rho} R^{\frac{n}{2(s+n)}} m^{-\frac{1}{3(1+n/s)}} \log
\frac{2}{\delta}.
$$ If $|Y| \leq M$ almost surely, then by taking $h =m^{\frac{1}{2(1+n/s)}}$, with confidence $1-\delta$,
$$\left\|\widetilde{f}_{\bf z} -f_\rho\right\|_{L^2_{\rho_X}} \leq C_{s, n, \rho} R^{\frac{n}{2(s+n)}}
m^{- \frac{1}{2 + 2n/s}} \log \frac{2}{\delta}.
$$
Here the constant $C_{s, n, \rho}$ is independent of $R$.
\end{example}

Compared to the analysis of least squares methods, our consistency
results for the MEE algorithm require a weaker condition by
allowing heavy tailed noise, while the convergence rates are
comparable but slightly worse than the optimal one
$O(m^{-\frac{1}{2 + n/s}})$. Further investigation of error
analysis for the MEE algorithm is required to achieve the optimal
rate, which is beyond the scope of this paper.

\section{Technical Difficulties in MEE and Novelties}\label{DiffNovelty}

The MEE algorithm (\ref{Etf}) involving sample pairs like
quadratic forms is different from most classical ERM learning
algorithms \cite{vapnik98, AB} constructed by sums of independent
random variables. But as done for some ranking algorithms
\cite{AgNi, CLV}, one can still follow the same line to define a
functional called generalization error or {\it information error}
(related to information potential defined on page 88 of
\cite{MEEbook}) associated with the windowing function $G$ over
the space of measurable functions on ${\mathcal X}$ as
\label{generalization error}
$$  {\mathcal E}^{(h)} (f) =\int_{\mathcal Z} \int_{\mathcal Z} - h^2 G \left(\frac{\left[\left(y-f(x)\right) -\left(y' -
f(x')\right)\right]^2}{2 h^2}\right) d \rho (x, y) d \rho(x', y').
$$ An essential barrier for our consistency analysis is an
observation made by numerical simulations \cite{ErdPri03, Silva}
and verified mathematically for Shannon's entropy in
\cite{ChenPri} that the regression function $f_\rho$ may not be a
minimizer of ${\mathcal E}^{(h)}$. It is totally different from
the classical least squares generalization error \label{least
squares generalization error} ${\mathcal E}^{ls}(f)
=\int_{\mathcal Z} (f(x) -y)^2 d \rho$ which satisfies a nice
identity ${\mathcal E}^{ls}(f) - {\mathcal E}^{ls}(f_{\rho})
=\|f-f_\rho\|^2_{L^2_{\rho_X}} \geq 0.$ This barrier leads to
three technical difficulties in our error analysis which will be
overcome by our novel approaches making full use of the special
feature that the MEE scaling parameter $h$ is large in this paper.

\subsection{Approximation of information error}

The first technical difficulty we meet in our mathematical analysis for MEE algorithm (\ref{Etf}) is the varying form
depending on the windowing function $G$. Our novel approach here is an approximation of the information error in terms of
the variance $\Var[f(X)-f_\rho(X)]$ when $h$ is large. This is achieved by showing that $\mathcal E^{(h)}$ is closely
related to the following {\it symmetrized least squares error} which has appeared in the literature of ranking algorithms
\cite{CLV, AgNi}.

\begin{definition}
The symmetrized least squares error is defined on the space $L^2_{\rho_X}$ by
\begin{equation}\label{slse}
{\mathcal E}^{sls} (f) = \int_{\mathcal Z} \int_{\mathcal Z} \left[\left(y-f(x)\right) -\left(y' - f(x')\right)\right]^2 d
\rho (x, y) d \rho(x', y'), \qquad f\in L^2_{\rho_X}.
\end{equation}
\end{definition}

To give the approximation of $\mathcal E^{(h)}$, we need a simpler form of ${\mathcal E}^{sls}$.

\begin{lemma}\label{ident}
If $\E[Y^2]< \infty$, then by denoting \label{constant rho}
$C_\rho =\int_{\mathcal Z} \left[y- f_\rho(x)\right]^2 d \rho$, we
have
\begin{equation}\label{slseId}
{\mathcal E}^{sls} (f) = 2 \Var[f(X)-f_\rho(X)] +2 C_\rho, \qquad \forall f \in L^2_{\rho_X}.
\end{equation}
\end{lemma}

\begin{proof} Recall that for two independent and identically distributed samples $\xi$ and $\xi'$ of a random variable, one has the identity
$$\E[(\xi-\xi')^2]= 2[\E(\xi-\E\xi)^2]=2\Var(\xi). $$
Then we have
$$
{\mathcal E}^{sls} (f)=  \E \left[\Big(\left(y-f(x)\right) -\left(y' - f(x')\right)\Big)^2\right] =  2 \Var[Y-f(X)]. $$ By
the definition $\E[Y|X]=f_\rho(X)$, it is easy to see that $C_\rho =\Var(Y-f_\rho (X))$ and the covariance between
$Y-f_\rho(X)$ and $f_\rho(X)-f(X)$ vanishes. So $\Var[Y-f(X)] =\Var(Y-f_\rho(X)) + \Var[f(X)-f_\rho(X)]$. This proves the
desired identity.
\end{proof}

We are in a position to present the approximation of $\mathcal E^{(h)}$ for which a large scaling parameter $h$ plays an
important role. Since ${\mathcal H}$ is a compact subset of $C({\mathcal X})$, we know that the number $\sup_{f\in
{\mathcal H}}\|f\|_\infty$ is finite.

\begin{lemma}\label{analysisG}
Under assumptions (\ref{assummoment}) and (\ref{assumpG}), for any essentially bounded measurable function $f$ on $X$, we
have
$$\left|{\mathcal E}^{(h)} (f) + h^2 G(0) - C_\rho - \Var[f(X)-f_\rho(X)] \right| \leq 5\cdot 2^7 C_G
\left((\E[|Y|^q])^{\frac{q^* +2}{q}} + \|f\|_\infty^{q^* +2}\right) h^{-q^*}.
$$
In particular,
$$\left|{\mathcal E}^{(h)} (f) +
h^2 G(0) - C_\rho - \Var[f(X)-f_\rho(X)]\right| \leq C_{{\mathcal
H}}' h^{-q^*}, \qquad \forall f \in {\mathcal H}, $$ where
$C_{{\mathcal H}}'$ is the constant depending on $\rho, G, q$ and
${\mathcal H}$ given by
$$ C_{{\mathcal H}}' =5\cdot 2^7 C_G
\left((\E[|Y|^q])^{(q^* +2)/q} + \left(\sup_{f\in {\mathcal
H}}\|f\|_\infty\right)^{q^* +2}\right). $$
\end{lemma}

\begin{proof}
Observe that $q^* +2 =\min\{q, 4\} \in (2, 4]$. By the Taylor expansion and the mean value theorem, we have
$$|G(t) - G(0) - G_+' (0) t| \leq \left\{\begin{array}{ll}
\frac{\|G''\|_{\infty}}{2}t^2 \leq \frac{\|G''\|_{\infty}}{2}t^{(q^* +2)/2}, & \hbox{if} \ 0 \leq t \leq 1, \\
2 \|G'\|_{\infty} t \leq 2 \|G'\|_{\infty} t^{(q^* +2)/2}, & \hbox{if} \ t >1. \end{array}\right. $$ So $|G(t) - G(0) -
G_+' (0) t| \leq \left(\frac{\|G''\|_{\infty}}{2} + 2 \|G'\|_{\infty}\right) t^{(q^* +2)/2}$ for all $t\geq 0$, and by
setting $t= \frac{\left[\left(y-f(x)\right) -\left(y' - f(x')\right)\right]^2}{2 h^2}$, we know that
\begin{eqnarray*} &&\left|{\mathcal E}^{(h)} (f) + h^2 G(0) +
\int_{\mathcal Z} \int_{\mathcal Z} G_+' (0) \frac{\left[\left(y-f(x)\right) -\left(y' - f(x')\right)\right]^2}{2} d \rho
(x, y) d \rho(x', y')\right| \\
&& \leq \left(\frac{\|G''\|_{\infty}}{2} + 2 \|G'\|_{\infty}\right) h^{-q^*} 2^{-(q^* +2)/2} \int_{\mathcal Z}
\int_{\mathcal Z} \left|\left(y-f(x)\right) -\left(y' - f(x')\right)\right|^{q^* +2} d \rho
(x, y) d \rho(x', y') \\
&& \leq \left(\frac{\|G''\|_{\infty}}{2} + 2 \|G'\|_{\infty}\right) h^{-q^*} 2^8 \left\{\int_{\mathcal Z} |y|^{q^* +2} d
\rho + \|f\|_\infty^{q^* +2}\right\}.
\end{eqnarray*}
This together with Lemma \ref{ident}, the normalization assumption
$G_+' (0) =-1$ and H\"older's inequality applied when $q>4$ proves
the desired bound and hence our conclusion.
\end{proof}

Applying Lemma \ref{analysisG} to a function $f \in {\mathcal H}$
and $f_\rho \in L^\infty_{\rho_X}$ yields the following fact on
the excess generalization error ${\mathcal E}^{(h)} (f) -
{\mathcal E}^{(h)} (f_\rho)$.

\begin{theorem}\label{compare}
Under assumptions (\ref{assummoment}) and (\ref{assumpG}), we have
$$\left|{\mathcal E}^{(h)} (f) - {\mathcal E}^{(h)} (f_\rho) - \Var[f(X)-f_\rho(X)]\right| \leq
C_{{\mathcal H}}'' h^{-q^*}, \qquad \forall f \in {\mathcal H}, $$ where $C_{{\mathcal H}}''$ is the constant depending on
$\rho, G, q$ and ${\mathcal H}$ given by
$$C_{{\mathcal H}}'' = 5 \cdot 2^8 C_G
\left((\E[|Y|^q])^{(q^* +2)/q} + \left(\sup_{f\in {\mathcal H}}\|f\|_\infty\right)^{q^* +2} + \|f_\rho\|_\infty^{q^*
+2}\right).
$$
\end{theorem}

\subsection{Functional minimizer and best approximation}

As $f_\rho$ may not be a minimizer of ${\mathcal E}^{(h)}$, the
second technical difficulty in our error analysis is the diversity
of two ways to define a {\it target function} in ${\mathcal H}$,
one to minimize the information error and the other to minimize
the variance $\Var[f(X)-f_\rho(X)]$. These possible candidates for
the target function are defined as \label{target fcn}
\label{approx fcn}
\begin{eqnarray}
&&f_{\mathcal H} :=\arg \min_{f\in {\mathcal H}} {\mathcal
E}^{(h)}
(f), \label{targetf} \\
&& f_{approx} := \arg \min_{f\in {\mathcal H}}
\Var[f(X)-f_\rho(X)]. \label{approx}
\end{eqnarray}
Our novelty to overcome the technical difficulty is to show that
when the MEE scaling parameter $h$ is large, these two functions
are actually very close.

\begin{theorem}\label{twocompare}
Under assumptions (\ref{assummoment}) and (\ref{assumpG}), we have
$${\mathcal E}^{(h)} (f_{approx}) \leq {\mathcal E}^{(h)} (f_{\mathcal
H}) + 2 C_{{\mathcal H}}'' h^{-q^*} $$ and
$$\Var[f_{\mathcal H}(X)-f_\rho(X)] \leq \Var[f_{approx}(X)-f_\rho(X)] +
2 C_{\mathcal H}'' h^{-q^*}. $$
\end{theorem}

\begin{proof} By Theorem \ref{compare} and the definitions of $f_{\mathcal H}$ and $f_{approx}$, we have
\begin{eqnarray*}
&& {\mathcal E}^{(h)} (f_{\mathcal H}) - {\mathcal E}^{(h)}
(f_\rho) \leq {\mathcal E}^{(h)} (f_{approx}) - {\mathcal E}^{(h)}
(f_\rho) \leq  \Var[f_{approx}(X)-f_\rho(X)] +
C_{{\mathcal H}}'' h^{-q^*} \\
&& \leq  \Var[f_{\mathcal H}(X)-f_\rho(X)] + C_{{\mathcal H}}''
h^{-q^*} \leq {\mathcal E}^{(h)} (f_{\mathcal H}) - {\mathcal
E}^{(h)} (f_\rho) + 2 C_{{\mathcal H}}'' h^{-q^*}
\\ && \leq  \Var[f_{approx}(X)-f_\rho(X)] +
3 C_{{\mathcal H}}'' h^{-q^*}.
\end{eqnarray*}
Then the desired inequalities follow.
\end{proof}

Moreover, Theorem \ref{compare} yields the following error
decomposition for our algorithm.

\begin{lemma}\label{errordec}
Under assumptions (\ref{assummoment}) and (\ref{assumpG}), we have
\begin{equation}\label{decomform}
\Var[f_{\bf z}(X) -f_\rho(X)] \leq \left\{{\mathcal E}^{(h)}
(f_{\bf z}) - {\mathcal E}^{(h)} (f_{\mathcal H})\right\} +
\Var[f_{approx}(X)-f_\rho(X)] + 2 C_{\mathcal H}'' h^{-q^*}.
\end{equation}
\end{lemma}

\begin{proof} By Theorem \ref{compare},
\begin{eqnarray*}
\Var[f_{\bf z}(X) -f_\rho(X)] && \leq {\mathcal E}^{(h)} (f_{\bf
z}) - {\mathcal E}^{(h)} (f_\rho) + C_{{\mathcal H}}'' h^{-q^*} \\
&& \leq \left\{{\mathcal E}^{(h)} (f_{\bf z}) - {\mathcal E}^{(h)} (f_{\mathcal H})\right\} + {\mathcal E}^{(h)}
(f_{\mathcal H}) - {\mathcal E}^{(h)} (f_\rho) + C_{{\mathcal H}}'' h^{-q^*}.
\end{eqnarray*}
Since $f_{approx} \in {\mathcal H}$, the definition of $f_{\mathcal H}$ tells us that
$${\mathcal E}^{(h)}
(f_{\mathcal H}) - {\mathcal E}^{(h)} (f_\rho) \leq {\mathcal
E}^{(h)} (f_{approx}) - {\mathcal E}^{(h)} (f_\rho). $$ Applying
Theorem \ref{compare} to the above bound implies
\begin{eqnarray*}
\Var[f_{\bf z}(X) -f_\rho(X)] \leq \left\{{\mathcal E}^{(h)}
(f_{\bf z}) - {\mathcal E}^{(h)} (f_{\mathcal H})\right\} +
\Var[f_{approx}(X)-f_\rho(X)] + 2 C_{{\mathcal H}}'' h^{-q^*}.
\end{eqnarray*}
Then desired error decomposition (\ref{decomform}) follows.
\end{proof}

{\it Error decomposition} has been a standard technique to analyze
least squares ERM regression algorithms \cite{AB, CZhou,
SmaleZhou, Ying}. In error decomposition (\ref{decomform}) for MEE
learning algorithm (\ref{Etf}), the first term on the right side
is the sample error, the second term
$\Var[f_{approx}(X)-f_\rho(X)]$ is the approximation error, while
the last extra term $2 C_{{\mathcal H}}'' h^{-q^*}$ is caused by
the Parzen windowing and is small when $h$ is large. The quantity
${\mathcal E}^{(h)} (f_{\bf z}) - {\mathcal E}^{(h)} (f_{\mathcal
H})$ of the sample error term will be bounded in the following
discussion.

\subsection{Error decomposition by U-statistics and special properties}

We shall decompose the sample error term ${\mathcal E}^{(h)} (f_{\bf z}) - {\mathcal E}^{(h)} (f_{\mathcal H})$ further by
means of U-statistics defined for $f\in {\mathcal H}$ and the sample ${\bf z}$ as
$$ V_f ({\bf z}) =\frac{1}{m (m-1)} \sum_{i=1}^m \sum_{j\not= i} U_f (z_i, z_j), $$
where $U_f$ is a kernel given with $z=(x, y), z' = (x', y') \in
{\mathcal Z}$ by \label{kernel for U statistics}
\begin{equation}\label{Ukernel}
U_f (z, z') =- h^2
G \left(\frac{\left[\left(y -f (x)\right) -\left(y' - f (x')\right)\right]^2}{2 h^2}\right) + h^2 G
\left(\frac{\left[\left(y -f_\rho (x)\right) -\left(y' - f_\rho (x')\right)\right]^2}{2 h^2}\right).
\end{equation}
It is easy to see
that $\E[V_f] ={\mathcal E}^{(h)} (f)- {\mathcal E}^{(h)} (f_\rho)$ and $U_f (z, z) =0$. Then
$${\mathcal E}^{(h)} (f_{\bf z}) - {\mathcal E}^{(h)} (f_{\mathcal H}) = \E\left[V_{f_{\bf z}}\right] - \E\left[V_{f_{\mathcal H}}\right]
= \E\left[V_{f_{\bf z}}\right] - V_{f_{\bf z}} + V_{f_{\bf z}} - V_{f_{\mathcal H}} + V_{f_{\mathcal H}} -
\E\left[V_{f_{\mathcal H}}\right]. $$ By the definition of $f_{\bf z}$, we have $V_{f_{\bf z}} - V_{f_{\mathcal H}} \leq
0$. Hence
\begin{equation}\label{SS}
{\mathcal E}^{(h)} (f_{\bf z}) - {\mathcal E}^{(h)} (f_{\mathcal
H}) \leq  \E\left[V_{f_{\bf z}}\right] - V_{f_{\bf z}} +
V_{f_{\mathcal H}} - \E\left[V_{f_{\mathcal H}}\right].
\end{equation}
The above bound will be estimated by a uniform ratio probability
inequality. A technical difficulty we meet here is the possibility
that $\E[V_f] ={\mathcal E}^{(h)} (f)- {\mathcal E}^{(h)}
(f_\rho)$ might be negative since $f_\rho$ may not be a minimizer
of ${\mathcal E}^{(h)}$. It is overcome by the following novel
observation which is an immediate consequence of Theorem
\ref{compare}.

\begin{lemma}\label{Restricteps}
Under assumptions (\ref{assummoment}) and (\ref{assumpG}), if $\varepsilon \geq C_{{\mathcal H}}'' h^{-q^*}$, then
\begin{equation}\label{secnovelty}
\E[V_f] + 2 \varepsilon \geq \E[V_f] + C_{{\mathcal H}}'' h^{-q^*}
+ \varepsilon \geq \Var[f(X)-f_\rho(X)] + \varepsilon \geq
\varepsilon, \qquad \forall f \in {\mathcal H}.
\end{equation}
\end{lemma}

\section{Sample Error Estimates}

In this section, we follow (\ref{SS}) and estimate the sample
error by a uniform ratio probability inequality based on the
following Hoeffding's probability inequality for U-statistics
\cite{Hoeffd}.

\begin{lemma}\label{Hoeffdinglemma}
If $U$ is a symmetric real-valued function on ${\mathcal Z} \times {\mathcal Z}$ satisfying $a\leq U (z, z') \leq b$ almost
surely and $\mathrm{var}[U]=\si^2$, then for any $\varepsilon>0,$
$$
\mathrm{Prob}\left\{\left|\frac{1}{m(m-1)}\sum_{i=1}^m \sum_{j\not= i} U(z_i, z_j)-\mathbb{E}[U]\right|
\geq\varepsilon\right\} \leq 2 \exp\left\{-\frac{(m-1)\varepsilon^2}{4\si^2+(4/3)(b-a)\varepsilon}\right\}.
$$
\end{lemma}

To apply Lemma \ref{Hoeffdinglemma} we need to bound $\sigma^2$ and $b-a$ for the kernel $U_f$ defined by (\ref{Ukernel}).
Our novelty for getting sharp bounds is to use a Taylor expansion involving a $C^2$ function $\widetilde{G}$ on $\RR$:
\begin{equation}\label{Taylorex}
\widetilde{G} (w) = \widetilde{G}(0) + \widetilde{G}' (0) w + \int_0^w (w-t) \widetilde{G}'' (t) d t, \qquad \forall w\in
\RR.
\end{equation}
Denote a constant $A_{\mathcal H}$ depending on $\rho, G, q$ and
${\mathcal H}$ as
$$
A_{\mathcal H} = 9 \cdot 2^8 C_G^2 \sup_{f\in {\mathcal H}}
\|f-f_\rho\|_\infty^{\frac{4}{q}} \left((\E[|Y|^q])^{\frac{2}{q}}
+ \|f_\rho\|_\infty^2 + \sup_{f\in {\mathcal H}}
\|f-f_\rho\|_\infty^2\right).
$$

\begin{lemma}\label{boundsingle}
Assume (\ref{assummoment}) and (\ref{assumpG}).

(a) For any $f, g\in {\mathcal H}$, we have
$$\left|U_f\right| \leq 4 C_G \|f- f_\rho\|_\infty h \ \hbox{and} \
\left|U_f -U_g\right| \leq 4 C_G \|f- g\|_\infty h $$ and
$$\mathrm{var}[U_f] \leq A_{\mathcal H} \left(\mathrm{var}[f(X) - f_\rho (X)]\right)^{(q-2)/q}.  $$

(b) If $|Y| \leq M$ almost surely for some constant $M >0$, then
we have almost surely
\begin{equation}\label{Ufunifboundspecial}
\left|U_f\right| \leq A_{\mathcal H}' \left|(f(x) - f_\rho (x)) - (f(x') - f_\rho (x'))\right|, \quad \forall f \in
{\mathcal H}
\end{equation}
and
\begin{equation}\label{Ufgunifboundspecial}
\left|U_f - U_g\right| \leq A_{\mathcal H}' \left|(f(x) - g (x)) -
(f(x') - g (x'))\right|, \quad \forall f, g \in {\mathcal H},
\end{equation}
where $A_{\mathcal H}'$ is a constant depending on $\rho, G$ and
${\mathcal H}$ given by
$$A_{\mathcal H}' = 36 C_G \left(M + \sup_{f\in {\mathcal H}}
\|f\|_\infty\right). $$
\end{lemma}

\begin{proof}
Define a function \label{intermediate function} $\widetilde{G}$ on
$\RR$ by
\begin{equation} \widetilde{G}(t) = G(t^2/2), \qquad t \in \RR.
\end{equation} We see that  $\widetilde{G} \in C^2 (\RR)$, $\widetilde{G}(0) =G(0)$, $\widetilde{G}'(0)=0$,
$\widetilde{G}'(t) = t G' (t^2/2)$ and $\widetilde{G}''(t) = G' (t^2/2) + t^2 G'' (t^2/2)$. Moreover,
\begin{eqnarray*}
U_f(z, z') &=&- h^2 \widetilde{G} \left(\frac{\left(y -f (x)\right) -\left(y' - f (x')\right)}{h}\right) +h^2 \widetilde{G}
\left(\frac{\left(y -f_\rho (x)\right) -\left(y' - f_\rho (x')\right)}{h}\right).
\end{eqnarray*}

(a) We apply the mean value theorem and see that $|U_f (z, z')|
\leq 2 h \|\widetilde{G}'\|_\infty \|f- f_\rho\|_\infty$. The
inequality for $|U_f -U_g|$ is obtained when $f_\rho$ is replaced
by $g$. Note that $\|\widetilde{G}'\|_\infty =\|t G'
(t^2/2)\|_\infty$. Then the bounds for $U_f$ and $U_f - U_g$ are
verified by noting $\|t G' (t^2/2)\|_\infty \leq 2 C_G$.

To bound the variance, we apply (\ref{Taylorex}) to the two points $w_1 =\frac{\left(y -f (x)\right) -\left(y' - f
(x')\right)}{h}$ and $w_2 =\frac{\left(y -f_\rho (x)\right) -\left(y' - f_\rho (x')\right)}{h}$. Writing $w_2 -t$ as $w_2 -
w_1 + w_1 -t$, we see from $\widetilde{G}'(0)=0$ that
\begin{eqnarray*}U_f(z, z') &=& h^2 \left(\widetilde{G}(w_2) - \widetilde{G}(w_1)\right) = h^2 \widetilde{G}'(0) (w_2 -w_1) \\
&& \quad + h^2 \int_0^{w_2} (w_2-t) \widetilde{G}'' (t) d t - h^2 \int_0^{w_1} (w_1 -t) \widetilde{G}'' (t) d
t \\
&=& h^2 \int_0^{w_2} (w_2-w_1) \widetilde{G}'' (t) d t + h^2 \int_{w_1}^{w_2} (w_1 -t) \widetilde{G}'' (t) d t.
\end{eqnarray*}
It follows that
\begin{eqnarray}
\left|U_f(z, z')\right|  &\leq& \|\widetilde{G}''\|_\infty \left|\left(y -f_\rho (x)\right) -\left(y' - f_\rho
(x')\right)\right| \ \left|\left(f(x) -f_\rho (x)\right) -\left(f(x') - f_\rho (x')\right)\right| \nonumber\\
&& + \|\widetilde{G}''\|_\infty \left|\left(f(x) -f_\rho (x)\right) -\left(f(x') - f_\rho (x')\right)\right|^2.
\label{Ufbound}
\end{eqnarray}

Since $\E[|Y|^q] < \infty$, we apply H\"older's inequality and see that
\begin{eqnarray*}
&& \int_{\mathcal Z} \int_{\mathcal Z}\left|\left(y -f_\rho (x)\right) -\left(y' - f_\rho (x')\right)\right|^2
\left|\left(f(x) -f_\rho (x)\right) -\left(f(x') - f_\rho (x')\right)\right|^2 d \rho (z) d \rho (z') \\
&& \leq \left\{\int_{\mathcal Z} \int_{\mathcal Z}\left|\left(y -f_\rho (x)\right) -\left(y' - f_\rho (x')\right)\right|^q
d \rho (z) d \rho (z')\right\}^{2/q} \\
&& \qquad  \left\{\int_{\mathcal Z} \int_{\mathcal Z} \left|\left(f(x) -f_\rho (x)\right) -\left(f(x') - f_\rho
(x')\right)\right|^{2q/(q-2)} d \rho (z) d \rho (z')\right\}^{1-2/q} \\
&& \leq \left\{4^{q+1}  (\E[|Y|^q] + \|f_\rho\|_\infty^q)\right\}^{2/q} \left\{\|f-f_\rho\|_\infty^{4/(q-2)} 2
\mathrm{var}[f(X) - f_\rho(X)]\right\}^{(q-2)/q}.
\end{eqnarray*}
Here we have separated the power index $2q/(q-2)$ into the sum of $4/(q-2)$ and $2$. Then
\begin{eqnarray*}
\mathrm{var}[U_f] &\leq& \E[U_f^2] \leq 2 \|\widetilde{G}''\|_\infty^2 2^{(5 q +3)/q} (\E[|Y|^q] +
\|f_\rho\|_\infty^q)^{2/q} \|f-f_\rho\|_\infty^{4/q}
\left(\mathrm{var}[f(X) - f_\rho(X)]\right)^{(q-2)/q}\\
&& + 2 \|\widetilde{G}''\|_\infty^2 4 \|f-f_\rho\|_\infty^{2} 2 \mathrm{var}[f(X) - f_\rho(X)].
\end{eqnarray*}
Hence the desired inequality holds true since
$\|\widetilde{G}''\|_\infty \leq \|G'\|_\infty + \|t^2 G''
(t^2/2)\|_\infty \leq 3 C_G$ and $\mathrm{var}[f(X) - f_\rho(X)]
\leq \|f-f_\rho\|_\infty^2$.

(b) If $|Y| \leq M$ almost surely for some constant $M >0$, then
we see from (\ref{Ufbound}) that almost surely $\left|U_f(z,
z')\right| \leq 4 \|\widetilde{G}''\|_\infty  (M +
\|f_\rho\|_\infty  + \|f- f_\rho\|_\infty) \left|\left(f(x)
-f_\rho (x)\right) -\left(f(x') - f_\rho (x')\right)\right|.$
Hence (\ref{Ufunifboundspecial}) holds true almost surely.
Replacing $f_\rho$ by $g$ in (\ref{Ufbound}), we see immediately
inequality (\ref{Ufgunifboundspecial}). The proof of Lemma
\ref{boundsingle} is complete.
\end{proof}

With the above preparation, we can now give the uniform ratio
probability inequality for U-statistics to estimate the sample
error, following methods in the learning theory literature
\cite{Haussler, Kolch, CZhou}.

\begin{lemma}\label{unformbd}
Assume (\ref{assummoment}), (\ref{assumpG}) and $\varepsilon \geq C_{\mathcal H}'' h^{-q^*}.$ Then we have
$$\mathrm{Prob} \left\{\sup_{f\in {\mathcal H}} \frac{\left|V_{f} -
\E[V_{f}]\right|}{(\E[V_{f}] + 2 \varepsilon)^{(q-2)/q}} > 4
\varepsilon^{2/q}\right\} \leq 2 {\mathcal N}\left({\mathcal H},
\frac{\varepsilon}{4 C_G h}\right) \exp\left\{-\frac{(m-1)
\varepsilon}{A_{\mathcal H}'' h}\right\},
$$
where $A_{\mathcal H}''$ is the constant given by
$$A_{\mathcal H}'' = 4 A_{\mathcal H}
(C_{\mathcal H}'')^{-2/q} + 12 C_G \sup_{f\in {\mathcal H}} \|f -
f_\rho\|_\infty.
$$
If $|Y| \leq M$ almost surely for some constant $M >0$, then we
have $$\mathrm{Prob} \left\{\sup_{f\in {\mathcal H}}
\frac{\left|V_{f} - \E[V_{f}]\right|}{\sqrt{\E[V_{f}] + 2
\varepsilon}} > 4 \sqrt{\varepsilon}\right\} \leq 2 {\mathcal
N}\left({\mathcal H}, \frac{\varepsilon}{2 A_{\mathcal H}'}\right)
\exp\left\{-\frac{(m-1) \varepsilon}{A_{\mathcal H}''}\right\},
$$
where $A_{\mathcal H}''$ is the constant given by
$$ A_{\mathcal H}'' = 8
A_{\mathcal H}' + 6  A_{\mathcal H}'\sup_{f\in {\mathcal H}} \|f -
f_\rho\|_\infty. $$
\end{lemma}

\begin{proof}
If $\|f-f_j\|_\infty \leq \frac{\varepsilon}{4 C_G h}$, Lemma \ref{boundsingle} (a) implies $|\E[V_f] - \E[V_{f_j}]| \leq
\varepsilon$ and $|V_f - V_{f_j}| \leq \varepsilon$ almost surely. These in connection with Lemma \ref{Restricteps} tell us
that
$$ \frac{\left|V_{f} - \E[V_{f}]\right|}{(\E[V_{f}] + 2 \varepsilon)^{(q-2)/q}} > 4 \varepsilon^{2/q}  \quad \Longrightarrow \quad
\frac{\left|V_{f_j} - \E[V_{f_j}]\right|}{(\E[V_{f_j}] + 2
\varepsilon)^{(q-2)/q}} > \varepsilon^{2/q}. $$ Thus by taking
$\{f_j\}_{j=1}^N$ to be an $\frac{\varepsilon}{4 C_G h}$ net of
the set ${\mathcal H}$ with $N$ being the covering number
${\mathcal N}\left({\mathcal H}, \frac{\varepsilon}{4 C_G
h}\right)$, we find
\begin{eqnarray*}
&&\mathrm{Prob} \left\{\sup_{f\in {\mathcal H}}
\frac{\left|V_{f} - \E[V_{f}]\right|}{(\E[V_{f}] + 2 \varepsilon)^{(q-2)/q}} > 4 \varepsilon^{2/q}\right\}
\leq \mathrm{Prob} \left\{ \sup_{j=1, \ldots, N}
\frac{\left|V_{f_j} - \E[V_{f_j}]\right|}{(\E[V_{f_j}] + 2 \varepsilon)^{(q-2)/q}} > \varepsilon^{2/q}\right\} \\
&& \leq \sum_{j=1, \ldots, N} \mathrm{Prob} \left\{\frac{\left|V_{f_j} - \E[V_{f_j}]\right|}{(\E[V_{f_j}] + 2
\varepsilon)^{(q-2)/q}} > \varepsilon^{2/q}\right\}.
\end{eqnarray*}

Fix $j\in \{1, \ldots, N\}$. Apply Lemma \ref{Hoeffdinglemma} to $U = U_{f_j}$ satisfying $\frac{1}{m(m-1)}\sum_{i=1}^m
\sum_{j\not= i} U(z_i, z_j)-\mathbb{E}[U] = V_{f_j} - \E[V_{f_j}]$. By the bounds for $|U_{f_j}|$ and
$\mathrm{var}[U_{f_j}]$ from Part (b) of Lemma \ref{boundsingle}, we know by taking $\widetilde{\varepsilon} =
\varepsilon^{2/q} (\E[V_{f_j}] + 2 \varepsilon)^{(q-2)/q}$ that
\begin{eqnarray*}&&\mathrm{Prob} \left\{\frac{\left|V_{f_j} - \E[V_{f_j}]\right|}{(\E[V_{f_j}] + 2
\varepsilon)^{(q-2)/q}} > \varepsilon^{2/q}\right\} = \mathrm{Prob} \left\{\left|V_{f_j} - \E[V_{f_j}]\right| >
\widetilde{\varepsilon}\right\} \\
&& \leq 2 \exp\left\{-\frac{(m-1)\widetilde{\varepsilon}^2}{4 A_{\mathcal H} \left(\mathrm{var}[f_j (X) - f_\rho
(X)]\right)^{(q-2)/q}+ 12 C_G \|f_j - f_\rho\|_\infty h \widetilde{\varepsilon}}\right\} \\
&& \leq 2 \exp\left\{-\frac{(m-1) \varepsilon^{4/q} (\E[V_{f_j}] +
2 \varepsilon)^{(q-2)/q}}{4 A_{\mathcal H}  + 12 C_G \|f_j -
f_\rho\|_\infty h \varepsilon^{2/q}}\right\},
\end{eqnarray*}
where in the last step we have used the important relation
(\ref{secnovelty}) to the function $f=f_j$ and bounded
$\left(\mathrm{var}[f_j (X) - f_\rho (X)]\right)^{(q-2)/q}$ by
$\left\{(\E[V_{f_j}] + 2 \varepsilon)\right\}^{(q-2)/q}.$ This
together with the notation $N= {\mathcal N}\left({\mathcal H},
\frac{\varepsilon}{4 C_G h}\right)$ and the inequality $\|f_{j}-
f_\rho\|_\infty \leq \sup_{f\in {\mathcal H}} \|f-f_\rho\|_\infty$
gives the first desired bound, where we have observed that
$\varepsilon \geq C_{\mathcal H}'' h^{-q^*}$ and $h \geq 1$ imply
$\varepsilon^{-2/q} \leq (C_{\mathcal H}'')^{-2/q} h$.

If $|Y| \leq M$ almost surely for some constant $M >0$, then we
follows the same line as in our above proof. According to Part (b)
of Lemma \ref{boundsingle}, we should replace $4 C_G h$ by $2
A_{\mathcal H}'$, $q$ by $4$, and bound the variance
$\mathrm{var}[U_{f_j}]$ by $2 A_{\mathcal H}' \mathrm{var}[f_j (X)
- f_\rho (X)] \leq 2 A_{\mathcal H}' (\E[V_{f_j}] + 2
\varepsilon)$. Then the desired estimate follows. The proof of
Lemma \ref{unformbd} is complete.
\end{proof}

We are in a position to bound the sample error. To unify the two
estimates in Lemma \ref{unformbd}, we denote $A_{\mathcal H}' = 2
C_G$ in the general case. For $m\in \NN, 0<\delta<1$, let
$\varepsilon_{m, \delta}$ be the smallest positive solution to the
inequality
\begin{equation}\label{Ch7eq}
 \log {\mathcal N}\left({\mathcal
H}, \frac{\varepsilon}{2 A_{\mathcal H}'}\right)
 -\frac{(m-1) \varepsilon}{A_{\mathcal H}''} \leq \log
 \frac{\delta}{2}.
 \end{equation}

\begin{prop}\label{bdS1}
Let $0< \delta <1, 0< \eta \leq 1$. Under assumptions (\ref{assummoment}) and (\ref{assumpG}), we have with confidence of
$1- \delta,$
$$\Var[f_{\bf z}(X) -f_\rho(X)] \leq (1+ \eta)
\Var[f_{approx}(X)-f_\rho(X)] + 12 \left(2 + 24^{(q-2)/2}\right)
\eta^{(2-q)/2} (h \varepsilon_{m, \delta} + 2 C_{\mathcal H}''
h^{-q^*}). $$ If $|Y| \leq M$ almost surely for some $M>0$, then
with confidence of $1- \delta,$ we have
$$\Var[f_{\bf z}(X) -f_\rho(X)] \leq (1+ \eta)
\Var[f_{approx}(X)-f_\rho(X)] + \frac{278}{\eta} (\varepsilon_{m,
\delta} + 2 C_{\mathcal H}'' h^{-2}). $$
\end{prop}

\begin{proof}
Denote $\tau= (q-2)/q$ and $\varepsilon_{m, \delta, h} =\max\{h
\varepsilon_{m, \delta}, C_{{\mathcal H}}'' h^{-q^*}\}$ in the
general case with some $q>2$, while $\tau= 1/2$ and
$\varepsilon_{m, \delta, h} =\max\{\varepsilon_{m, \delta},
C_{{\mathcal H}}'' h^{-2}\}$ when $|Y| \leq M$ almost surely. Then
by Lemma \ref{unformbd}, we know that with confidence $1-\delta$,
there holds $$ \sup_{f\in {\mathcal H}} \frac{\left|V_{f} -
\E[V_{f}]\right|}{(\E[V_{f}] + 2 \varepsilon_{m, \delta,
h})^{\tau}} \leq 4 \varepsilon_{m, \delta, h}^{1-\tau}$$ which
implies
\begin{eqnarray*}
&& \E\left[V_{f_{\bf z}}\right] - V_{f_{\bf z}} + V_{f_{\mathcal
H}} - \E\left[V_{f_{\mathcal H}}\right] \leq 4 \varepsilon_{m,
\delta, h}^{1-\tau} (\E[V_{f_{\bf z}}] + 2 \varepsilon_{m, \delta,
h})^{\tau} + 4 \varepsilon_{m, \delta, h}^{1-\tau}
(\E[V_{f_{\mathcal H}}] + 2 \varepsilon_{m, \delta, h})^{\tau}.
\end{eqnarray*}
This together with Lemma \ref{errordec} and (\ref{SS}) yields
\begin{equation}\label{interdecom}
\Var[f_{\bf z}(X) -f_\rho(X)] \leq 4 {\mathcal S}  + 16
\varepsilon_{m, \delta, h} + \Var[f_{approx}(X)-f_\rho(X)] + 2
C_{\mathcal H}'' h^{-q^*},
\end{equation}
where
$$ {\mathcal S} := \varepsilon_{m,
\delta, h}^{1-\tau} (\E[V_{f_{\bf z}}])^{\tau} +  \varepsilon_{m,
\delta, h}^{1-\tau} (\E[V_{f_{\mathcal H}}])^{\tau} =
\left(\frac{24}{\eta}\right)^\tau \varepsilon_{m, \delta,
h}^{1-\tau} \left(\frac{\eta}{24}\E[V_{f_{\bf z}}]\right)^{\tau} +
\left(\frac{12}{\eta}\right)^\tau \varepsilon_{m, \delta,
h}^{1-\tau} \left(\frac{\eta}{12}\E[V_{f_{\mathcal
H}}]\right)^{\tau}.
$$ Now we apply Young's inequality
$$ a \cdot b \leq (1-\tau) a^{1/(1-\tau)} + \tau b^{1/\tau},
\qquad a, b \geq 0 $$ and find
$${\mathcal S} \leq
\left(\frac{24}{\eta}\right)^{\tau/(1-\tau)} \varepsilon_{m,
\delta, h} + \frac{\eta}{24}\E[V_{f_{\bf z}}] +
\left(\frac{12}{\eta}\right)^{\tau/(1-\tau)} \varepsilon_{m,
\delta, h} + \frac{\eta}{12}\E[V_{f_{\mathcal H}}].
$$
Combining this with (\ref{interdecom}), Theorem \ref{compare} and
the identity $\E[V_f] ={\mathcal E}^{(h)} (f)- {\mathcal E}^{(h)}
(f_\rho)$ gives
$$\Var[f_{\bf z}(X) -f_\rho(X)] \leq \frac{\eta}{6} \Var[f_{\bf z}(X) -f_\rho(X)]
+ (1+ \frac{\eta}{3}) \Var[f_{approx}(X)-f_\rho(X)] + {\mathcal
S}',$$ where ${\mathcal S}' := (16 + 8 (24/\eta)^{\tau/(1-\tau)})
\varepsilon_{m, \delta, h} + 3 C_{\mathcal H}'' h^{-q^*}$. Since
$1/(1-\frac{\eta}{6}) \leq 1+\frac{\eta}{3}$ and
$(1+\frac{\eta}{3})^2 \leq 1+\eta$, we see that
$$\Var[f_{\bf z}(X) -f_\rho(X)] \leq (1+ \eta) \Var[f_{approx}(X)-f_\rho(X)] + \frac{4}{3}{\mathcal
S}'. $$ Then the desired estimates follow, and the proposition is
proved.
\end{proof}

\section{Proof of Main Results}\label{mainproof}

We are now in a position to prove our main results stated in
Section \ref{mainresults}.

\proof[Proof of Theorem \ref{consistthm}] Recall ${\mathcal
D}_{\mathcal H} (f_\rho) = \Var[f_{approx}(X)-f_\rho(X)]$. Take
$\eta =\min\{\epsilon/(3{\mathcal D}_{\mathcal H} (f_\rho)), 1\}$.
Then $\eta \Var[f_{approx}(X)-f_\rho(X)] \leq \epsilon/3$. Now we
take $$ h_{\epsilon, \delta} =\left(72 \left(2 +
24^{(q-2)/2}\right) \eta^{(2-q)/2} C_{\mathcal
H}''/\epsilon\right)^{1/q^*}. $$ Set $\widetilde{\varepsilon}
:=\epsilon/\left(36 \left(2 + 24^{(q-2)/2}\right)
\eta^{(2-q)/2}\right)$. We choose
$$m_{\epsilon, \delta}(h) = \frac{h A_{\mathcal H}''}{\widetilde{\varepsilon}}\left(\log {\mathcal N}\left({\mathcal
H}, \frac{\widetilde{\varepsilon}}{2 h A_{\mathcal H}'}\right) -
\log \frac{\delta}{2}\right) +1. $$ With this choice, we know that
whenever $m \geq m_{\epsilon, \delta}(h)$, the solution
$\varepsilon_{m, \delta}$ to inequality (\ref{Ch7eq}) satisfies
$\varepsilon_{m, \delta} \leq \widetilde{\varepsilon}/h$.
Combining all the above estimates and Proposition \ref{bdS1}, we
see that whenever $h \geq h_{\epsilon, \delta}$ and $m \geq
m_{\epsilon, \delta}(h)$, error bound (\ref{consistbd}) holds true
with confidence $1-\delta$. This proves Theorem \ref{consistthm}.
\qed

\proof[Proof of Theorem \ref{mainresult}] We apply Proposition
\ref{bdS1}. By covering number condition (\ref{covercond}), we
know that $\varepsilon_{m, \delta}$ is bounded by
$\widetilde{\varepsilon}_{m, \delta}$, the smallest positive
solution to the inequality
$$
A_p \left(\frac{2 A_{\mathcal H}'}{\varepsilon}\right)^p
 -\frac{(m-1) \varepsilon}{A_{\mathcal H}''} \leq \log
 \frac{\delta}{2}.
$$
This inequality written as $\varepsilon^{1+p} - \frac{A_{\mathcal
H}''}{m-1}\log \frac{2}{\delta} \varepsilon^p - A_p \left(2
A_{\mathcal H}'\right)^p \frac{A_{\mathcal H}''}{m-1} \geq 0$ is
well understood in learning theory (e.g. \cite{CZhou}) and its
solution can be bounded as
$$\widetilde{\varepsilon}_{m, \delta} \leq \max\left\{2 \frac{A_{\mathcal
H}''}{m-1}\log \frac{2}{\delta}, \left(2 A_p A_{\mathcal H}'' (2
A_{\mathcal H}')^p\right)^{1/(1+p)}
(m-1)^{-\frac{1}{1+p}}\right\}.
$$

If $\E[|Y|^q]< \infty$ for some $q>2$, then the first part of
Proposition \ref{bdS1} verifies (\ref{mainbd}) with the constant
$\widetilde{C}_{\mathcal H}$ given by $$ \widetilde{C}_{\mathcal
H}= 24 \left(2 + 24^{(q-2)/2}\right) \left(2 A_{\mathcal H}'' +
\left(2 A_p A_{\mathcal H}'' (2 A_{\mathcal
H}')^p\right)^{1/(1+p)} + 2 C_{\mathcal H}''\right). $$

If $|Y| \leq M$ almost surely for some $M>0$, then the second part
of Proposition \ref{bdS1} proves (\ref{mainbdM}) with the constant
$\widetilde{C}_{\mathcal H}$ given by $$ \widetilde{C}_{\mathcal
H}= 278 \left(2 A_{\mathcal H}'' + \left(2 A_p A_{\mathcal H}'' (2
A_{\mathcal H}')^p\right)^{1/(1+p)} + 2 C_{\mathcal H}''\right).
$$
This completes the proof of Theorem \ref{mainresult}. \qed

\proof[Proof of Theorem \ref{meanesti}] Note $\left|\frac{1}{m} \sum_{i=1}^m \left[f(x_i) - \pi_{\sqrt{m}} (y_i)\right] -
\frac{1}{m} \sum_{i=1}^m \left[g(x_i) - \pi_{\sqrt{m}} (y_i)\right]\right| \leq \|f-g\|_\infty$ and $\left|\E [f(X)
-\pi_{\sqrt{m}} (Y)] - \E [g(X) -\pi_{\sqrt{m}} (Y)]\right| \leq \|f-g\|_\infty$. So by taking $\{f_j\}_{j=1}^N$ to be an
$\frac{\varepsilon}{4}$ net of the set ${\mathcal H}$ with $N = {\mathcal N}\left({\mathcal H},
\frac{\varepsilon}{4}\right)$, we know that for each $f \in {\mathcal H}$ there is some $j \in \{1, \ldots, N\}$ such that
$\|f-f_j\|_\infty \leq \frac{\varepsilon}{4}$. Hence
\begin{eqnarray*} && \left|\frac{1}{m} \sum_{i=1}^m \left[f(x_i) -
\pi_{\sqrt{m}} (y_i)\right] - \E [f(X) -\pi_{\sqrt{m}} (Y)]\right| >
\varepsilon \\
\Longrightarrow &&\left|\frac{1}{m} \sum_{i=1}^m \left[f_j(x_i) -
\pi_{\sqrt{m}} (y_i)\right] - \E [f_j(X) -\pi_{\sqrt{m}} (Y)]\right|
> \frac{\varepsilon}{2}.
\end{eqnarray*}
It follows that
\begin{eqnarray*}
&&\mathrm{Prob} \left\{\sup_{f\in {\mathcal H}} \left|\frac{1}{m}
\sum_{i=1}^m \left[f(x_i) - \pi_{\sqrt{m}} (y_i)\right] - \E [f(X)
-\pi_{\sqrt{m}} (Y)]\right| >
\varepsilon\right\} \\
\leq && \mathrm{Prob} \left\{ \sup_{j=1, \ldots, N}
\left|\frac{1}{m} \sum_{i=1}^m \left[f_j(x_i) - \pi_{\sqrt{m}}
(y_i)\right] - \E [f_j(X)
-\pi_{\sqrt{m}} (Y)]\right| > \frac{\varepsilon}{2}\right\} \\
\leq && \sum_{j=1}^{N} \mathrm{Prob} \left\{\left|\frac{1}{m}
\sum_{i=1}^m \left[f_j(x_i) - \pi_{\sqrt{m}} (y_i)\right] - \E
[f_j(X) -\pi_{\sqrt{m}} (Y)]\right| > \frac{\varepsilon}{2}\right\}.
\end{eqnarray*}
For each fixed $j\in \{1, \ldots, N\}$, we apply the classical
Bernstein probability inequality to the random variable $\xi =
f_j(X) -\pi_{\sqrt{m}} (Y)$ on $(Z, \rho)$ bounded by $\widetilde{M}
= \sup_{f\in {\mathcal H}} \|f\|_\infty + \sqrt{m}$ with variance
$\sigma^2 (\xi) \leq \E [|f_j(X) -\pi_{\sqrt{m}} (Y)|^2] \leq 2
\sup_{f\in {\mathcal H}} \|f\|_\infty^2 + 2 \E
[|Y|^2]=:\sigma^2_{\mathcal H}$ and know that
\begin{eqnarray*}&&\mathrm{Prob} \left\{\left|\frac{1}{m} \sum_{i=1}^m \left[f_j(x_i) -
\pi_{\sqrt{m}} (y_i)\right] - \E [f_j(X) -\pi_{\sqrt{m}}
(Y)]\right| > \frac{\varepsilon}{2}\right\} \\
&& \leq 2 \exp\left\{-\frac{m (\varepsilon/2)^2}{\frac{2}{3}
\widetilde{M}\varepsilon/2 + 2 \sigma^2 (\xi)}\right\} \leq  2
\exp\left\{-\frac{m \varepsilon^2}{\frac{4}{3}
\widetilde{M}\varepsilon + 8 \sigma^2_{\mathcal H}}\right\}.
\end{eqnarray*}
The above argument together with covering number condition (\ref{covercond}) yields
\begin{eqnarray*}
&&\mathrm{Prob} \left\{\sup_{f\in {\mathcal H}} \left|\frac{1}{m}
\sum_{i=1}^m \left[f(x_i) - \pi_{\sqrt{m}} (y_i)\right] - \E [f(X)
-\pi_{\sqrt{m}} (Y)]\right| >
\varepsilon\right\} \\
\leq && 2 N \exp\left\{-\frac{m \varepsilon^2}{\frac{4}{3} \widetilde{M}\varepsilon + 8 \sigma^2_{\mathcal H}}\right\} \leq
2 \exp\left\{A_p \left(\frac{4}{\varepsilon}\right)^{p} -\frac{m \varepsilon^2}{\frac{4}{3} \widetilde{M}\varepsilon + 8
\sigma^2_{\mathcal H}}\right\}.
\end{eqnarray*}
Bounding the right-hand side above by $\delta$ is equivalent to the inequality
$$ \varepsilon^{2+p} - \frac{4}{3m} \widetilde{M} \log \frac{2}{\delta} \varepsilon^{1 + p}
- \frac{8}{m} \sigma^2_{\mathcal H} \log \frac{2}{\delta}
\varepsilon^{p} - \frac{A_p 4^{p}}{m} \geq 0. $$ By taking
$\widetilde{\varepsilon}_{m, \delta}$ to be the smallest solution
to the above inequality, we see from \cite{CZhou} as in the proof
of Theorem \ref{mainresult} that with confidence at least
$1-\delta$,
\begin{eqnarray*}
&&\sup_{f\in {\mathcal H}} \left|\frac{1}{m} \sum_{i=1}^m
\left[f(x_i) - \pi_{\sqrt{m}} (y_i)\right] - \E [f(X)
-\pi_{\sqrt{m}} (Y)]\right| \\
&& \leq \widetilde{\varepsilon}_{m, \delta} \leq
\max\left\{\frac{4 \widetilde{M}}{m}\log \frac{2}{\delta},
\sqrt{\frac{24 \sigma^2_{\mathcal H}}{m} \log \frac{2}{\delta}},
\left(\frac{A_p 4^{p}}{m} \right)^{\frac{1}{2+p}}\right\} \\
&& \leq \left\{7\sup_{f\in {\mathcal H}} \|f\|_\infty + 4 + 7 \sqrt{\E [|Y|^2]} + 4 A_p^{\frac{1}{2+p}}\right\}
m^{-\frac{1}{2+p}} \log \frac{2}{\delta}.
\end{eqnarray*}
Moreover, since $\pi_{\sqrt{m}} (y) -y =0$ for $|y| \leq \sqrt{m}$ while $|\pi_{\sqrt{m}} (y) - y| \leq |y| \leq
\frac{|y|^2}{\sqrt{m}}$ for $|y|
> \sqrt{m}$, we know that
\begin{eqnarray*}
&& \left|\E [\pi_{\sqrt{m}} (Y)] - \E[f_\rho (X)]\right| =
\left|\int_X \int_Y \pi_{\sqrt{m}} (y) -y d \rho(y|x) d
\rho_X(x)\right| \\
&&= \left|\int_X \int_{|y|>\sqrt{m}} \pi_{\sqrt{m}} (y) -y d \rho(y|x) d \rho_X(x)\right| \leq \int_X \int_{|y|>\sqrt{m}}
\frac{|y|^2}{\sqrt{m}} d \rho(y|x) d \rho_X(x) \leq \frac{\E[|Y|^2]}{\sqrt{m}}.
\end{eqnarray*}
Therefore, (\ref{mainmeanesti}) holds with confidence at least
$1-\delta$. The proof of Theorem \ref{meanesti} is complete. \qed

\section{Conclusion and Discussion}

In this paper we have proved the consistency of an MEE algorithm associated with R\'enyi's entropy of order 2 by letting
the scaling parameter $h$ in the kernel density estimator tends to infinity at an appropriate rate. This result explains
the effectiveness of the MEE principle in empirical applications where the parameter $h$ is required to be large enough
before smaller values are tuned. However, the motivation of the MEE principle is to minimize error entropies approximately,
and requires small $h$ for the kernel density estimator to converge to the true probability density function. Therefore,
our consistency result seems surprising.

As far as we know, our result is the first rigorous consistency result for MEE algorithms. There are many open questions in
mathematical analysis of MEE algorithms. For instance, can MEE algorithm (\ref{Etf}) be consistent by taking $h\to 0$? Can
one carry out error analysis for the MEE algorithm if Shannon's entropy or R\'enyi's entropy of order $\alpha \not= 2$ is
used? How can we establish error analysis for other learning settings such as those with non-identical sampling processes
\cite{SmaleZhou, Hu}? These questions required further research and will be our future topics.

It might be helpful to understand our theoretical results by
relating MEE algorithms to ranking algorithms. Note that MEE
algorithm (\ref{Etf}) essentially minimizes the empirical version
of the information error which, according to our study in Section
2, differs from the symmetrized least squares error used in some
ranking algorithms by an extra term which vanishes when
$h\to\infty.$ Our study may shed some light on analysis of some
ranking algorithms.

 {\begin{table}  \caption{NOTATIONS} \begin{center}
  \begin{tabular}{|l|l|l|}
  \hline
  notation&  meaning &  pages\\
  \hline
 $p_E$ & probability density function of a random variable $E$ &
 \pageref{pdfE} \\\hline
 $H_S(E)$ & Shannon's entropy of a random variable $E$ & \pageref{Shannon entropy} \\\hline
 $H_{R, \alpha} (E)$ & R\'enyi's entropy of order $\alpha$ & \pageref{Renyi entropy} \\\hline
 $X$ & explanatory variable for learning & \pageref{explanatory
variable} \\\hline
 $Y$ & response variable for learning & \pageref{response
 variable} \\\hline
$E=Y-f(X)$ & error random variable associated with a predictor
$f(X)$ & \pageref{random variable} \\\hline $H_{R}(E)$ & R\'enyi's
entropy of order $\alpha =2$ & \pageref{Renyi entropy 2} \\\hline
$\bz=\{(x_i, y_i)\}_{i=1}^{m}$ & a sample for learning &
\pageref{sample} \\\hline $G$ & windowing function &
\pageref{windowing function}, \pageref{windowing function MEE},
\pageref{windowing function constant}
\\\hline $h$ & MEE scaling parameter & \pageref{MEE scaling
parameter}, \pageref{MEE scaling parameter 2}
\\\hline
$\widehat{p}_E$ &  Parzen windowing approximation of $p_E$ &
\pageref{Parzen windowing} \\\hline $\widehat{H_S}$ & empirical
Shannon entropy & \pageref{empirical Shannon entropy} \\\hline
 $\widehat{H_R}$ & empirical R\'enyi's entropy of order $2$ & \pageref{empirical Renyi entropy} \\\hline
$f_\rho$ & the regression function of $\rho$ &
 \pageref{regression function} \\\hline
$f_{\bf z}$ & output function of the MEE learning algorithm (\ref{Etf}) & \pageref{MEE algorithm} \\\hline ${\mathcal H}$ &
the hypothesis space for the ERM algorithm & \pageref{hypothesis space} \\\hline $\Var$ & the variance of a random variable
& \pageref{variance} \\\hline $q, q^*=\min\{q-2, 2\}$ & power indices in condition (\ref{assummoment}) for $\E[|Y|^q]
<\infty$ & \pageref{power q}
\\\hline
$C_G$ & constant for decay condition (\ref{assumpG}) of $G$ &
\pageref{decay constant}
\\\hline
${\mathcal D}_{\mathcal H} (f_\rho)$ & approximation error of the
pair $({\mathcal H}, \rho)$ & \pageref{approximation error}
\\\hline
${\mathcal N}\left({\mathcal H}, \varepsilon\right)$ & covering
number of the hypothesis space ${\mathcal H}$ & \pageref{covering
number}
\\\hline
$p$ & power index for covering number condition (\ref{covercond})
& \pageref{covering index}
\\\hline
$\pi_{\sqrt{m}}$ & projection onto the closed interval
$[-\sqrt{m}, \sqrt{m}]$ & \pageref{projection}
\\\hline
$\widetilde{f}_{\bf z}$ & estimator of $f_\rho$ &
\pageref{regression estimator}
\\\hline
${\mathcal E}^{(h)} (f)$ & generalization error associated with
$G$ and $h$ & \pageref{generalization error}
\\\hline
${\mathcal E}^{ls}(f)$ & least squares generalization error
${\mathcal E}^{ls}(f) =\int_{\mathcal Z} (f(x) -y)^2 d \rho$ &
\pageref{least squares generalization error}
\\\hline
$C_\rho$ & constant $C_\rho =\int_{\mathcal Z} \left[y-
f_\rho(x)\right]^2 d \rho$ associated with $\rho$ &
\pageref{constant rho}
\\\hline
$f_{\mathcal H}$ & minimizer of ${\mathcal E}^{(h)} (f)$ in
${\mathcal H}$ & \pageref{target fcn}
\\\hline
$f_{approx}$ & minimizer of $\Var[f(X)-f_\rho(X)]$ in ${\mathcal
H}$  & \pageref{approx fcn}
\\\hline
$U_f$ & kernel for the U statistics $V_f$ & \pageref{kernel for U
statistics}
\\\hline
$\widetilde{G}$ & an intermediate function defined by $\widetilde{G}(t) = G(t^2/2)$
 & \pageref{intermediate function}
\\\hline
\hline
 \end{tabular}
 \end{center}
\end{table}}

\end{document}